\documentclass[runningheads]{llncs}
\usepackage{graphicx}
\usepackage{algorithm}
\usepackage{algpseudocode}
\algnewcommand\algorithmicforeach{\textbf{for each}}
\algdef{S}[FOR]{ForEach}[1]{\algorithmicforeach\ #1\ \algorithmicdo}
\algnewcommand\AndAlg{\textbf{and}}

\usepackage{amssymb,mathrsfs,amsmath,amsfonts,graphicx,pdfpages,booktabs,hyperref}
\usepackage{tikz, tikz-cd}
\usepackage{makecell}
\usepackage{multirow}

\newcommand{\maps}{\colon}

\makeatletter
\def\endthebibliography{%
  \def\@noitemerr{\@latex@warning{Empty `thebibliography' environment}}%
  \endlist
}
\makeatother

\usepackage{subfiles}
\usepackage{stackengine}
\usepackage{mathtools}
\usepackage[utf8]{inputenc}
\usepackage{color}

\usepackage{comment}

\newcommand*\circled[1]{\tikz[baseline=(char.base)]{
            \node[shape=circle,draw,inner sep=2pt] (char) {#1};}}
            
\newcommand*\groundedcircled[1]{\tikz[baseline=(char.base)]{
            \node[shape=circle,draw,inner sep=2pt, fill=gray!40] (char) {#1};}}

\definecolor{purple(x11)}{rgb}{0.8, 0, 0.8}

\usetikzlibrary{arrows,positioning,fit,matrix,shapes.geometric,external}

\newcommand*\pgfdeclareanchoralias[3]{%
  \expandafter\def\csname pgf@anchor@#1@#3\expandafter\endcsname
     \expandafter{\csname pgf@anchor@#1@#2\endcsname}}

\tikzset{
    circnode/.style={
      circle, draw=red, very thin, outer sep=0.025em, minimum size=2em,
      fill=red, text centered},
    integral/.style={
      circle, draw=black, very thick, outer sep=0.025em,
      minimum size=2em, fill=blue!5, text centered},
    multiply/.style={
      circle, draw=black, very thick, outer sep=0.025em,
      minimum size=2em, fill=blue!5, text centered},
    zero/.style={
      circle, draw=black, very thick, minimum size=0.15cm, fill=black,
      inner sep=0, outer sep=0},
    bang/.style={
      circle, draw=black, very thick, minimum size=0.15cm, fill=green!10,
      inner sep=0, outer sep=0},
    delta/.style={
      regular polygon, regular polygon sides=3, minimum size=0.4cm, inner
      sep=0, outer sep=0.025em, draw=black, very thick, fill=green!10},
    codelta/.style={
      regular polygon, regular polygon sides=3, shape border rotate=180, minimum size=0.4cm,
      inner sep=0, outer sep=0.025em, draw=black, very thick, fill=green!10},
    plus/.style={
      regular polygon, regular polygon sides=3, shape border rotate=180, minimum size=0.4cm,
      inner sep = 0, outer sep=0.025em, draw=black, very thick, fill=black},
    coplus/.style={
      regular polygon, regular polygon sides=3, minimum size=0.4cm,
      inner sep = 0, outer sep=0.025em, draw=black, very thick, fill=black},
    rect/.style={
      rectangle, minimum height=1cm, minimum width=.5cm,
      draw=black, inner sep=0.2em, outer sep=0.025em, text centered},
    bigcirc/.style={
      circle, draw=black, very thick, text width=1.6em, outer sep=0.025em,
      minimum height=1.6em, fill=blue!5, text centered}
}

\tikzstyle{tri}=[regular polygon,regular polygon sides=3,shape border rotate=1
80,fill=none,draw=black]

\pgfdeclareanchoralias{regular polygon}{corner 2}{left copy}
\pgfdeclareanchoralias{regular polygon}{corner 3}{right copy}
\pgfdeclareanchoralias{regular polygon}{corner 3}{addend}
\pgfdeclareanchoralias{regular polygon}{corner 2}{summand}

\usetikzlibrary{positioning}
\definecolor {processblue}{cmyk}{0.9,0.5,0,0}

\tikzstyle{simple}=[-,line width=2.000]
\tikzstyle{arrow}=[-,postaction={decorate},decoration={markings,mark=at position .5 with {\arrow{>}}},line width=1.100]
\pgfdeclarelayer{edgelayer}
\pgfdeclarelayer{nodelayer}
\pgfdeclarelayer{bg}
\pgfsetlayers{bg,edgelayer,nodelayer,main}

\tikzstyle{none}=[inner sep=-1pt]
\tikzstyle{species}=[circle,fill=none,draw=black,scale=0.75]
\tikzstyle{transition}=[rectangle,fill=none,draw=black,scale=1.15]
\tikzstyle{empty}=[circle,fill=none, draw=none]
\tikzstyle{inputdot}=[circle,fill=black,draw=black, scale=.5]
\tikzstyle{dot}=[circle,fill=black,draw=black]
\tikzstyle{bounding}=[circle,dashed, fill=none,draw=black, scale=9.00]
\tikzstyle{triplebounding}=[circle,dashed, fill=none,draw=black, scale=30.00]
\tikzstyle{simple}=[-,draw=black,line width=1.000]
\tikzstyle{inarrow}=[-,draw=black,postaction={decorate},decoration={markings,mark=at position .5 with {\arrow{>}}},line width=1.000]
\tikzstyle{tick}=[-,draw=black,postaction={decorate},decoration={markings,mark=at position .5 with {\draw (0,-0.1) -- (0,0.1);}},line width=1.000]
\tikzstyle{inputarrow}=[->,draw=black, shorten >=.05cm]

\usetikzlibrary{backgrounds,circuits,circuits.ee.IEC,shapes,fit,matrix}

\tikzstyle{tri}=[regular polygon,regular polygon sides=3,shape border rotate=1
80,fill=none,draw=black]

\tikzstyle{simple}=[-,line width=2.000]
\tikzstyle{arrow}=[-,postaction={decorate},decoration={markings,mark=at position .5 with {\arrow{>}}},line width=1.100]
\pgfdeclarelayer{edgelayer}
\pgfdeclarelayer{nodelayer}
\pgfsetlayers{edgelayer,nodelayer,main}

\tikzstyle{none}=[inner sep=-1pt]

\definecolor{lblue}{rgb}{0,250,255}

\tikzstyle{species}=[circle,fill=yellow,draw=black,scale=1]
\tikzstyle{transition}=[rectangle,fill=lblue,draw=black,scale=1]
\tikzstyle{morphism}=[rectangle,fill=pink,draw=black,scale=1]

\usetikzlibrary{arrows,snakes,automata,petri,backgrounds,circuits,circuits.ee.IEC,shapes,fit,matrix}
\tikzstyle{place}=[circle,thick,draw=blue!75,fill=blue!20,minimum size=6mm]
\tikzstyle{red place}=[place,draw=red!75,fill=red!20]
\tikzstyle{transition}=[rectangle,thick,draw=black!75,
  			  fill=black!20,minimum size=4mm]

\begin{document}
\title{String Diagrams for Assembly Planning\thanks{All authors contributed equally to this work while at Siemens}}
%
%
\author{Jade Master \inst{1}\orcidID{0000-0003-1970-6030} \and
Evan Patterson\inst{2}\orcidID{0000-0002-8600-949X} \and
Shahin Yousfi\inst{3} \and Arquimedes Canedo\inst{3}\orcidID{0000-0003-3506-6563} }
\authorrunning{J. Master, E. Patterson et al.}
%
\institute{University of California Riverside, Riverside CA 92507,  \email{jmast003@ucr.edu} \and
Stanford University, 450 Serra Mall, Stanford, CA 94305, \email{evan@epatters.org}
 \and Siemens Corporate Technology, 755 College Rd E, Princeton, NJ 08540 \email{arquimedes.canedo@siemens.com}}

\maketitle              
\begin{abstract}
Assembly planning is a difficult problem for companies. Many disciplines such as design, planning, scheduling, and manufacturing execution need to be carefully engineered and coordinated to create successful product assembly plans. Recent research in the field of \textit{design for assembly} has proposed new methodologies to design product structures in such a way that their assembly is easier. However, present assembly planning approaches lack the engineering tool support to capture all the constraints associated to assembly planning in a unified manner. This paper proposes \textsc{CompositionalPlanning}, a string diagram based framework for assembly planning. In the proposed framework, string diagrams and their compositional properties serve as the foundation for an engineering tool where CAD designs interact with planning and scheduling algorithms to automatically create high-quality assembly plans. These assembly plans are then executed in simulation to measure their performance and to visualize their key build characteristics. We demonstrate the versatility of this approach in the LEGO assembly domain. We developed two reference LEGO CAD models that are processed by \textsc{CompositionalPlanning}'s algorithmic pipeline. We compare sequential and parallel assembly plans in a Minecraft simulation and show that the time-to-build performance can be optimized by our algorithms.

\end{abstract}
\section{Introduction}

Today, mass customization of products such as automobiles and consumer electronics is forcing companies to provide a very large product variety to address the diverse customer requirements. Digital manufacturing technologies make it possible to accommodate mass customization during product design and manufacturing. For example, parametric designs in computer aided design (CAD) software allow for the specification of configurable products,  and computer aided manufacturing (CAM) algorithms allow for the fabrication of products on different machines. Unfortunately, there is very limited engineering tool support for product assembly planning. Although some companies use \textit{design for assembly} (DfA)~\cite{Kretschmer} and \textit{design for manufacturing and assembly} (DfMA)~\cite{FAVI} methodologies that attempt to develop product structures that facilitate their assembly, their implementation is ad-hoc. Therefore, assembly planning is a task that is loosely coupled to the rest of the digital manufacturing pipeline. The objective of this paper is to open new avenues for \textit{interoperable} assembly planning that is tightly coupled to the upstream design activities, and the downstream assembly tasks.  

String diagrams are a powerful graphical calculus for reasoning in category theory \cite{selinger}. String diagrams have also proven useful in many other domains. They have been shown to provide a mathematically sound graphical language in domains including linguistics \cite{coecke}, systems engineering \cite{baez}, and computer science \cite{master}. Generally, string diagrams represent processes which require and produce resources. Assembly planning is the discipline of understanding how to optimally chain assembly processes together to craft a whole product from separate parts~\cite{Ghandi}. Thus, string diagrams are a natural tool for formulating assembly planning problems and constructing their solutions.

To demonstrate this thesis we show that string diagrams can be used to build construction schedules for various LEGO models. From each LEGO 3D CAD file we generate a \textit{connectivity graph} where the nodes represent LEGO pieces and the edges indicate that they are connected in the final model. Given a hierarchical clustering of this connectivity graph, we generate a \textit{construction plan} represented by string diagrams which is hierarchical, compositional, and interpretable. Using the formalism of string diagrams, complex sub-assemblies can be \textit{black boxed} into larger string diagrams. Having this hierarchical structure allows us to manipulate or adapt our plan at a desired level of abstraction. Furthermore, there is a categorical formalism that enables schedules to be generated from these string diagrams. We use topological sorting, Girvan-Newman, and Leiden algorithms to generate assembly plans and schedules with different properties. Finally, we use Minecraft \cite{minecraftgame} as a simulator to validate the resulting schedules and measure their time-to-build performance.

In this paper, we demonstrate the versatility of this approach with a framework, \textsc{CompositionalPlanning}, that provides a new way of talking about assembly planning. Our category theoretic interpretation provides a flexible mathematical foundation that allows for an end-to-end demonstration from CAD design to assembly simulation. The contributions of this paper are the following:
\begin{itemize}
     \item we show how string diagrams are an intuitive yet mathematically sound language to represent an assembly planning domain.
     \item as large string diagrams can be tedious and cumbersome for humans to work with, our framework automates the creation of large string diagrams and thus eliminates the overhead traditionally associated with them.
     \item a novel algorithm that converts string diagrams to \textit{expressions} that result in highly parallel assembly plans 
    \item a Minecraft based simulation environment modification or ``mod'' to execute LEGO assembly plans.
     \item we publish the \textsc{CompositonalPlanning} framework as a Julia~\cite{Julia} package for others to reproduce and build upon our work\footnote{\textsc{CompositionalPlanning} - https://github.com/CompositionalPlanning/}.
\end{itemize}

\section{Related Work}

Assembly planning problems and approaches have been widely investigated. A comprehensive survey of them can be found in ~\cite{Ghandi}. In ~\cite{zha1998integrated}, a survey of  assembly design and planning systems is presented. In particular, the \textit{Assembly Sequence Planning} (ASP) problem that we target in this paper is an NP-hard problem. ASP's goal is to find a collision-free sequence of assembly operations that put together individual parts given the geometry of the final product and the relative positions of parts in the final product. ASP is considered a combinatorial problem and therefore representations of the space of possible sequences has been an active area of research~\cite{Jimenez}. While various representations ranging from AND-OR graphs ~\cite{de1990and} to Petri nets ~\cite{rosell2004assembly} have been proposed, we are the first to study ASP under a category theoretic framework. In this paper, we show that category theory provides both an \textit{explicit} and an \textit{implicit} representation on assembly sequences. On the one hand, string diagrams provide an explicit ASP representation of partial and full assemblies. On the other hand, expressions provide us mathematical soundness and rigor as they implicitly encode precedence assembly relationships.

Although this paper focuses on LEGO and Minecraft models, the \textsc{CompositionalPlanning} framework can be easily adapted to other domains. In 3D CAD modeling, for example, parts are composed into assemblies in a similar way to how LEGO pieces compose into LEGO models. The specification of the parts themselves and their composition of assemblies is defined in a CAD file that is analogous to a LEGO 3D CAD file. Therefore, the first step would be to create a custom parser for a specific CAD file format to extract a connectivity graph from the information about parts and their composition. Given that there are many CAD file formats in the market today~\cite{NISTCAD}, the scalability of \textsc{CompositionalPlanning} in the CAD domain is dependent on the availability of custom parsers.

Related work in automated assembly planning illustrates how this field is highly fragmented. Most researchers develop custom assembly planning solutions that work for specific products and processes. In~\cite{Boschert}, the authors present a system to address the assembly planning of multi-variant products in modular production systems. The product requirements and their feasible assembly orders are modeled in a directed graph referred to as the \textit{Augmented Assembly Priority Plan} (AAPP). The AAPP encodes how two initial subassemblies are joined together to form a new subassembly through a value adding task such as ``assembly'' or ``screwing''. From the authors' description, the AAPP can be mapped as a string diagram to leverage the planning capabilities of our \textsc{CompositionalPlanning} framework. Similarly, the authors in~\cite{Pintzos} present a system to generate \textit{assembly precedence graphs} from CAD files. Similar to our results, they show that an assembly precedence graph contains all the valid sequences of an assembly. Their assembly precedence graphs correspond to our connectivity graph, and their assembly sequences correspond to our plans and schedules. There are two important differences compared to our work: (a) their means of user-interaction and visualization are spreadsheets instead of graphs, and (b) their implementation is based on a proprietary CAD software.


\section{Categorical Assembly Planning Framework}
Our \textsc{CompositionalPlanning} framework, shown in Figure~\ref{fig:pipeline}, consists of five reusable software components. The first step is to infer a connectivity graph from the CAD model that describes how the different parts come together in terms of their geometry (e.g., orientation) and assembly operations (e.g., snap, glue, weld, insert). This step is necessary because most CAD models list quantities of all parts and detailed structural functions in relation to a finished product, and this information does not directly map to their assembly.

\vspace{-0.2in}
\begin{figure*}[!ht]
\center{\includegraphics[width=1\textwidth]{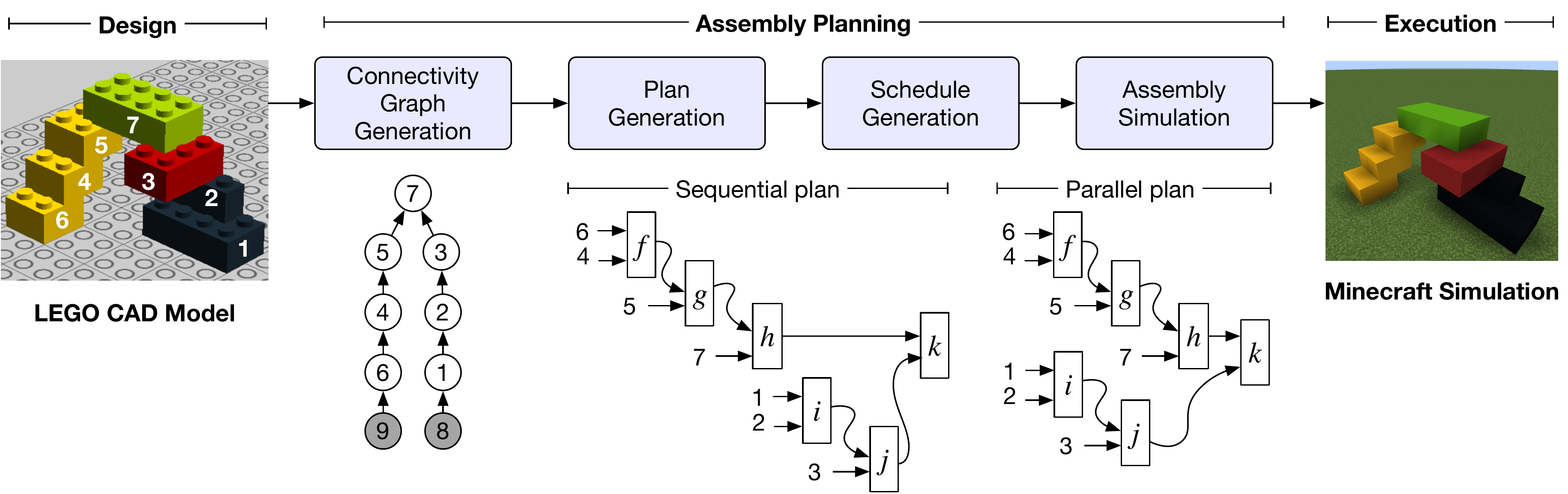}}
\caption{\label{fig:pipeline}\textsc{CompositionalPlanning} framework pipeline.}
\end{figure*}
\vspace{-0.2in}

The second step consists of plan generation~\cite{FoxKempf} -- described by different string diagrams -- that have different properties of order and parallelism. These plans, although expressed by different expressions of string diagrams, generate the same LEGO model as they: (i) bring an initial world to a goal world using a set of assembly operators, and (ii) minimally impose ordering constraints.

The third step generates a schedule using a plan and a detailed knowledge of the execution environment (e.g., number of workers, machines). The job of the schedule generation~\cite{FoxKempf} is to impose further ordering constraints on the assembly operator application to achieve a robust (e.g., against failures) and time-efficient execution of the assembly task (e.g., time-to-assembly).

The fourth and final step consists of executing the schedule in a simulator to visualize the assembly task, and to generate performance metrics. Although this paper focuses on LEGO CAD models as an input, and Minecraft simulations as an output, our components are domain agnostic, and they can be easily adapted for use in other domains.

\subsection{Connectivity Diagram Generation from CAD}
 \textsc{CompositionalPlanning} parses the text-based LDraw files~\cite{ldraw} and automatically builds the connectivity diagram. Every LEGO model has a unique connectivity diagram that describes how the pieces are connected to each other. LDraw files~\cite{ldraw} describe all the bricks in the model by type (e.g., 2$\times$2, 2$\times$4), color, center coordinates ($x$, $y$, $z$), and a 3$\times$3 rotation matrix. This LDraw file represents the bill of materials (BOM) of the LEGO model and does not contain any information about the connectivity of the bricks. Therefore, the first step is to parse this information from the LDraw file $f$ and generate a list of LegoObjects. The second step  is to create a directed graph and add a node for every object in the LegoObjects list. Note that these nodes are not yet connected by edges.


Vertical stacking is the most common operation with LEGO bricks as shown in the example in Figure~\ref{fig:pipeline}. Therefore, the third step in our framework is to infer the vertical connectivity in the LEGO CAD model. In LDraw's coordinate system $-y$ is ``up''. Therefore, two bricks are connected if: (a) the top face of one brick has the same $y$ coordinate as the other brick's down face ($ a.top\_ycoord() == b.bottom\_ycoord()$); (b) and their boxes (defined by $(x, z)$ center coordinates and the brick's $width$ and $length$) intersect ($(abs(a.x - b.x) * 2 < (a.length + b.length))~and~(abs(a.z - b.z) * 2 < (a.width + b.width))$). For every connected pair of objects we create an edge from $a$ to $b$ in the connectivity diagram. Other less common operations such as horizontal stacking, and operations involving other LEGO pieces such as pegs are left for future work. However, the connectivity inference would follow a similar principle as the one described above.

The fourth step consists of grounding all nodes in the connectivity diagram that do not have any predecessors. Having explicit ground nodes helps the processing of the connectivity diagram by the following algorithms. As an illustrative example consider the LEGO model and its connectivity diagram shown in Figure~\ref{fig:pipeline}. It consists of seven bricks sequentially numbered from $\circled{1}, ..., \circled{7}$. In addition, our algorithm also includes the ``ground'' nodes to facilitate the model construction using the base build plate. In this example, the ground nodes $\groundedcircled{9}$ and $\groundedcircled{8}$ connected to $\circled{6}$ and $\circled{1}$, respectively.

Extending our framework beyond LEGO would require new parsers to read CAD file formats, and new inference algorithms to derive the connectivity between parts.

\subsection{String Diagrams}
String diagrams are diagrams where resources are represented by strings (wires) and processes are represented by boxes. For example a process which snaps a peg into a hole is represented by
\begin{center}
\begin{tikzpicture}[remember picture,font={\fontsize{9}{14.4}},container/.style={inner sep=0},every path/.style={solid, line width=0.4pt},decoration={markings, mark=at position 0.5 with {\arrow{Stealth}}},execute at begin node=$,execute at end node=$]
  \node[container] (n) {
    \begin{tikzpicture}
      \node[minimum height=4.5em] (n1) {};
      \node[draw,solid,inner sep=0.333em,rectangle,rounded corners,minimum height=4.5em,right=2em of n1)] (n2) {\mathrm{snap}};
      \draw[postaction={decorate}] ($(n1.center)+(0,1.25em)$) to[out=0,in=180] node[above=0.25em,midway] {\mathrm{L1}} ($(n2.west)+(0,1.25em)$);
      \draw[postaction={decorate}] ($(n1.center)+(0,-1.25em)$) to[out=0,in=180] node[above=0.25em,midway] {\mathrm{L2}} ($(n2.west)+(0,-1.25em)$);
      \node[minimum height=2.0em,right=2em of n2)] (n3) {};
      \draw[postaction={decorate}] (n2.east) to[out=0,in=180] node[above=0.25em,midway] {\mathrm{base}} (n3.center);
    \end{tikzpicture}
  };
\end{tikzpicture}
\end{center}
String diagrams can be composed in sequence
\begin{center}
    \begin{tikzpicture}[remember picture,font={\fontsize{9}{14.4}},container/.style={inner sep=0},every path/.style={solid, line width=0.4pt},decoration={markings, mark=at position 0.5 with {\arrow{Stealth}}},execute at begin node=$,execute at end node=$]
  \node[container] (n) {
    \begin{tikzpicture}
      \node[minimum height=7.0em] (n1) {};
      \node[container,right=2em of n1)] (n2) {
        \begin{tikzpicture}
          \node[draw,solid,inner sep=0.333em,rectangle,rounded corners,minimum height=4.5em] (n21) {\mathrm{snap}};
          \node[minimum height=2.0em,below=0.5em of n21] (n22) {};
        \end{tikzpicture}
      };
      \draw[postaction={decorate}] ($(n1.center)+(0,2.5em)$) to[out=0,in=180] node[above=0.25em,midway] {L1} ($(n21.west)+(0,1.25em)$);
      \draw[postaction={decorate}] ($(n1.center)+(0,0.0em)$) to[out=0,in=180] node[above=0.25em,midway] {L2} ($(n21.west)+(0,-1.25em)$);
      \draw[postaction={decorate}] ($(n1.center)+(0,-2.5em)$) to[out=0,in=180] node[above=0.25em,midway] {L3} (n22.center);
      \node[draw,solid,inner sep=0.333em,rectangle,rounded corners,minimum height=4.5em,right=2em of n2)] (n3) {\mathrm{snap}};
      \draw[postaction={decorate}] (n21.east) to[out=0,in=180] node[above=0.25em,midway] {\mathrm{base}} ($(n3.west)+(0,1.25em)$);
      \draw[postaction={decorate}] (n22.center) to[out=0,in=180] node[above=0.25em,midway] {} ($(n3.west)+(0,-1.25em)$);
      \node[minimum height=2.0em,right=2em of n3)] (n4) {};
      \draw[postaction={decorate}] (n3.east) to[out=0,in=180] node[above=0.25em,midway] {C} (n4.center);
    \end{tikzpicture}
  };
\end{tikzpicture}
\end{center}
indicating that the processes must be performed in sequence. String diagrams can also be composed in parallel
\begin{center}
    \begin{tikzpicture}[remember picture,font={\fontsize{9}{14.4}},container/.style={inner sep=0},every path/.style={solid, line width=0.4pt},decoration={markings, mark=at position 0.5 with {\arrow{Stealth}}},execute at begin node=$,execute at end node=$]
  \node[container] (n) {
    \begin{tikzpicture}
      \node[minimum height=4.5em] (n1) {};
      \node[container,right=2em of n1)] (n2) {
        \begin{tikzpicture}
          \node[draw,solid,inner sep=0.333em,rectangle,rounded corners,minimum height=2.0em] (n21) {\mathrm{buildcolumn}};
          \node[draw,solid,inner sep=0.333em,rectangle,rounded corners,minimum height=2.0em,below=0.5em of n21] (n22) {\mathrm{buildroof}};
        \end{tikzpicture}
      };
      \draw[postaction={decorate}] ($(n1.center)+(0,1.25em)$) to[out=0,in=180] node[above=0.25em,midway] {A1} (n21.west);
      \draw[postaction={decorate}] ($(n1.center)+(0,-1.25em)$) to[out=0,in=180] node[above=0.25em,midway] {A2} (n22.west);
      \node[minimum height=4.5em,right=2em of n2)] (n3) {};
      \draw[postaction={decorate}] (n21.east) to[out=0,in=180] node[above=0.25em,midway] {B1} ($(n3.center)+(0,1.25em)$);
      \draw[postaction={decorate}] (n22.east) to[out=0,in=180] node[above=0.25em,midway] {B2} ($(n3.center)+(0,-1.25em)$);
    \end{tikzpicture}
  };
\end{tikzpicture}
\end{center}
indicating that the order in which the tasks are performed does not matter. String diagrams also have algebraic expressions called morphisms. For example, the first example has a corresponding algebraic expression given by
\[ \mathrm{snap} \maps \mathrm{L1} \otimes \mathrm{L2} \to \mathrm{base} \]
Note that this expression is both functional and typed. This makes \textsc{CompositionalPlanning} an efficient and type-safe framework. In a similar way, the expressions which string diagrams represent can be composed in sequence and parallel using the two operations  \begin{equation}
(f \maps x \to y, g \maps y \to z) \mapsto f \cdot g \maps x \to z
\end{equation}
\begin{equation} (f \maps x \to y, f' \maps x' \to y') \mapsto f \otimes f' \maps x \otimes x' \to y \otimes y' \end{equation}
Rather complicated expressions can be built using these operations, e.g., see the string diagrams in Figure \ref{fig:sequential_plans}. A natural question to ask is when two expressions correspond to the same string diagram. The answer to this question is essential to understanding how planning domain dependencies represented in string diagrams can be algebraically manipulated. It turns out that if the algebraic expressions satisfy the right set of axioms, then the way that a string diagram is drawn is independent of the expression it generates. A structure of algebraic expressions satisfying these axioms is a well-known structure in category theory called a symmetric monoidal category. In~\cite{jstreet} it is shown that string diagrams unambiguously represent morphisms in a given symmetric monoidal category. The following theorem ensures the soundness of string diagrams in the LEGO assembly domain.

\begin{theorem}
Let $G=(E,V)$ be a simple graph whose nodes $V$ represent pieces of a LEGO model and whose edges $E$ indicate a connection in the completed structure. Then, there is a symmetric monoidal category where:
\begin{itemize}
    \item an object is finite tensor product $X_1 \otimes X_2 \ldots X_n$ of subsets of $V$ i.e. all possible tensors of subassemblies.
    \item A morphism $f \colon X_1 \otimes X_2 \ldots \otimes X_n \to Y_1 \otimes Y_2 \ldots Y_n$ is a construction plan which turns the subassemblies $X_1 \otimes X_2  \ldots \otimes X_n $  into the subassemblies $Y_1 \otimes Y_2 \ldots \otimes Y_n$ using only the joins allowed by the edges of $G$.
    \item The composite $g \cdot f$ represents the construction plan where $f$ and $g$ are performed in sequence.
    \item The tensor product $g \otimes f$ represents the construction plan where $f$ and $g$ are performed in parallel.
\end{itemize}
\end{theorem}
\begin{proof}
Symmetric monoidal categories can be freely generated from the data of a Petri net. In ``On the Category of Petri Net Computations", Sassone showed that for a Petri net $P$, there is a strict symmetric monoidal category $\mathcal{Q}[P]$ whose objects are finite strings of places in your Petri net and whose morphisms correspond to \emph{strongly concatenable processes} \cite{SassoneStrong}. These are sequences of events which can occur using the transitions of your Petri net in sequence and in parallel. 

Recall that a Petri net is a tuple $(T,P,s,t)$ where
\begin{itemize}
    \item $T$ is a finite set of events which can occur,
    \item $P$ is a finite set of available resources, 
    \item $s \maps T \to P^{\oplus}$ is a function from events to multisets of resources indicating which resources are required for each event and,
    \item $t \maps T \to P^{\oplus}$ is a function from events to multisets of resources indicating which resources are produced by each event.
\end{itemize}
To construct the desired symmetric monoidal category, we set $T$ equal to $\mathcal{P}(E)$, the set of subets of edges in $G$ and set $P$ equal to $\mathcal{P}(V)$ the set of subsets of nodes in $G$ i.e. all possible sub-assemblies of the LEGO model. Define $s \maps \mathcal{P} (E) \to \mathcal{P}(V)^{\oplus}$
by the rule 
\[\{(x_1,y_1), (x_2,y_2), \ldots (x_n,y_n) \} \mapsto \{x_1,x_2,\ldots,x_n\} + \{y_1,y_2,\ldots,y_n\}. \]
where $+$ indicates the occurrence of both subsets in the multiset $\mathcal{P}(V)^{\oplus}$. Define $t \maps \mathcal{P} (E) \to \mathcal{P}(V)^{\oplus}$ by the rule
\[\{(x_1,y_1), (x_2,y_2), \ldots (x_n,y_n) \} \mapsto \{x_1,x_2,\ldots,x_n\} \cup \{y_1,y_2,\ldots,y_n\}. \]
The symmetric monoidal category in the theorem statement is obtained by taking the category of strongly concatenable processes on this Petri net.
\end{proof}

The next section describes how the string diagrams of these symmetric monoidal categories can be leveraged to produce construction plans.

\subsection{Planning}\label{sec:plans}

A LEGO CAD model and its connectivity graph describe an object in its final, assembled state, but not how to assemble it. A \textit{plan} consists of step-by-step instructions on how to assemble the atomic parts (LEGO bricks) into the desired object. In general, there are many possible plans for assembling the same object, corresponding to different ways of forming intermediate sub-assemblies.

For us, plans are string diagrams. In such a diagram, the strings represent sub-assemblies and the boxes represent operations of joining together sub-assemblies to make a larger sub-assembly. Formally, a sub-assembly is a subset of the atomic parts, interpreted as being assembled. Each join operation takes two disjoint sub-assemblies $A$ and $B$ as inputs and produces a single output, the union $A \cup B$. A plan is a string diagram that takes all the singleton sets (sub-assemblies consisting of a single part) as inputs and produces as output the set of all parts (full assembly). Although every plan is valid at this level of description, not every plan will be physically feasible.

As proof of concept, we implemented two simple algorithms for assembly planning. In the \textit{sequential algorithm}, we topologically order the edges of the connectivity graph. That is, we first topologically order the nodes, where a topological ordering is any total ordering consistent with the directed edges. Then we lexicographically order the edges, viewed as ordered pairs of nodes. For each edge, taken in this order, we join the two sub-assemblies containing the source and target, if they are distinct; otherwise, we do nothing. We continue in this way until all the edges have been exhausted, at which point the object is fully assembled. When the connectivity graph is a path graph, the sequential plan is the obvious plan that joins the parts together one-at-a-time.

In the \textit{parallel algorithm}, we create more opportunities for parallelism by partitioning the connectivity graph into components, making plans on each component, and then treating these plans as black boxes in a higher-level plan. This meta-algorithm has several knobs to tune, and it can be applied recursively. To partition the graph, we can apply any community detection algorithm that finds non-overlapping communities. In our experiments, we use the Girvan-Newman algorithm \cite{girvan2002} and a variant of the Louvain method \cite{blondel2008}, the Leiden algorithm \cite{traag2019}. We perform only one level of partitioning and we do sequential planning within each partition and also to assemble the resulting sub-plans. The use of community detection to find opportunities for parallelism is a heuristic, but works well in our experiments.

\subsection{Scheduling}

A plan, in the form of a string diagram, says what steps to perform and how the steps depend on each other. A \textit{schedule} extends the information in a plan by assigning the steps a definite order; formally, a schedule is any linear extension of the topological ordering of the operations (boxes) in the plan. For simplicity, we take a resource-agnostic view of scheduling, in which the number of workers is unknown at the time of planning and scheduling. The aim in scheduling is therefore to maximize the opportunities for parallelism, given the constraints imposed by the plan \cite{rosenberg2016}.

Our scheduling algorithms have two major phases. First, we create a syntactic expression representing the plan. In general, a single string diagram can be represented by many different expressions; we construct one of them. We then linearize the expression, using a simple recursive algorithm, to obtain a schedule.

A small example will illustrate the relationship between string diagrams and syntactic expressions. Consider the string diagram shown in Figure~\ref{fig:fghk}, the composite of $f$ and $g$ in parallel with composite of $h$ and $k$. We can represent this diagram by either of the expressions $(f \cdot g) \otimes (h \cdot k)$ (read ``$f$ then $g$, and $h$ then $k$'') and $(f \otimes h) \cdot (g \otimes k)$ (read ``$f$ and $h$, then $g$ and $k$'').

\vspace{-0.2in}
\begin{figure}[!htb]
\center{\includegraphics[width=0.75\textwidth]{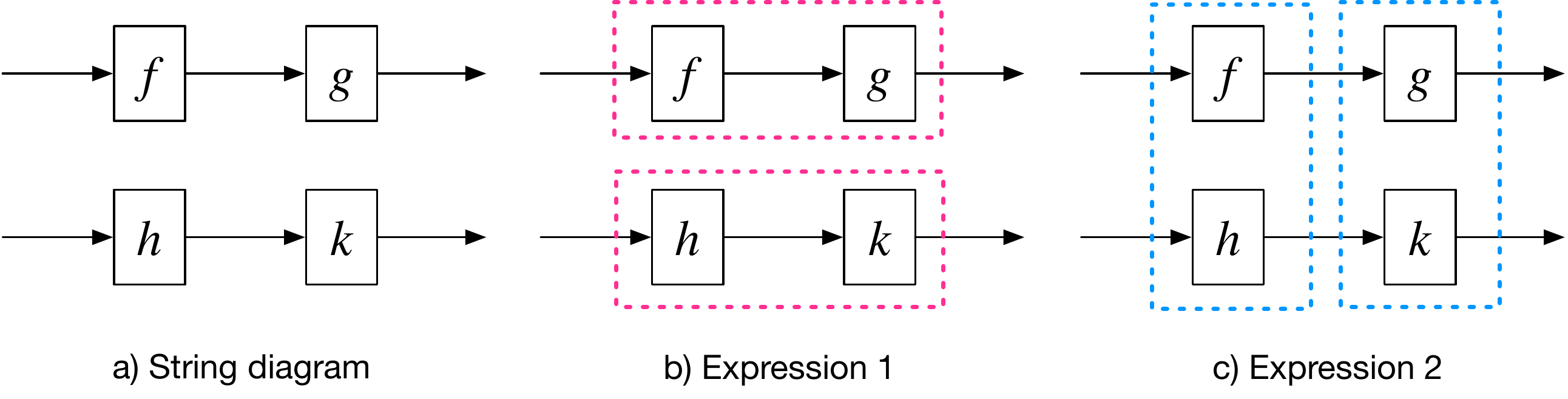}
\caption{\label{fig:fghk} Relationship between string diagrams and syntactic expressions}}
\end{figure}
\vspace{-0.2in}

We have developed an algorithm to find an expression for any string diagram representing a morphism in a symmetric monoidal category. As the algorithm is fairly elaborate, we will not digress to present it carefully, except to say that it is inspired by existing algorithms that recognize in a DAG, or reduce a DAG to, a series-parallel digraph \cite{valdes1982,mitchell2004}.

The second phase of scheduling is more straightforward. Having formed an expression for the plan, we schedule the plan by recursively linearizing the expression tree. Given a composition $f \cdot g$ , we simply concatenate the schedules for $f$ and $g$. Given a product $f \otimes g$, we \textit{interleave} the schedules for $f$ and $g$, meaning that we take the first element of the $f$-schedule, then the first element of the $g$-schedule, then the second element of the $f$-schedule, and so on, until the both schedules have been exhausted. For example, both of the above expressions $(f \cdot g) \otimes (h \cdot k)$ and $(f \otimes h) \cdot (g \otimes k)$ yield the same schedule $(f,h,g,k)$. Note that the ordering of the monoidal products affects the schedule,  so that $f \otimes g$ yields a different schedule than $g \otimes f$. This procedure can be seen as a special case of an existing algorithm for optimally scheduling series-parallel digraphs \cite{cordasco2014}.

\subsection{Simulation}

Minecraft is an immensely popular 3D open-world video game where players can build their own structures~\cite{Duncan}. The game world and most of its elements are made of different kinds of blocks. These blocks can be used to create structures of any complexity. This versatility makes Minecraft a good fit to represent the CAD model and to simulate their assembly process. It has already proven to be a well-suited simulation tool in other robotics domains~\cite{aluru2015minecraft}. To execute the schedule generated by our framework we extended Minecraft by a new mod. With this ``mod'' we can simulate the whole assembly process of the CAD model described by the schedule. The simulation not only provides us with a comprehensible visual representation of the process, but also allows us to quantify execution time and worker occupancy of different schedules.

Our open source Minecraft ``mod'' executes the assembly operations in the correct order as dictated by the generated schedule. The geometric CAD information is encoded in the connectivity diagram and it is parsed by our mod. Each LEGO brick is represented by a single or multiple Minecraft blocks. The schedule and the operations specified in it determine which and how many bricks can be connected to each other per step. In addition to the precedence constraints encoded in the schedule, we can define a number of \textit{workers}. Each worker is allowed to perform a single operation per step. Operations are dispatched to workers in the order they appear in the schedule. Only operations for which a worker is available can be executed. Hence, the level of parallelism of the schedule and the number of workers determine the time it takes to complete the assembly. Disjoint sub-assemblies are assembled in their own area respectively until they are connected to each other forming a new (sub-)assembly. The assembly process is finished when all of the schedule’s operations have been completed by the available workers.

\section{Results}
To validate the  \textsc{CompositionalPlanning}'s pipeline we designed the two LEGO CAD models shown in Figure~\ref{fig:cad_models}(a-b). The design objective was to have two LEGO models of around 100 bricks each with a rich set of features and a few human-intuitive sub-assemblies to validate our approach. The Columns model (Figure~\ref{fig:cad_models}(a)) is inspired by roman temples and consists of 77 bricks. This model consists of four column sub-assemblies, each composed of 12 vertically stacked 2$\times$2 bricks each. The roof sub-assembly consists of two pairs of support beams, each arranged in a stair configuration supporting a flat roof. The House model (Figure~\ref{fig:cad_models}(b)) consists of 86 bricks. The house foundation sub-assembly consists of eight 4$\times$10 green bricks and supports the house sub-assembly and an electric pole. The house sub-assembly consists of four non-identical walls. We designed each wall using different compositions of bricks (i.e., 1$\times$1, 1$\times$8, 1$\times$10, 1$\times$2$\times$2, etc.) that result in different connectivity features as highlighted by the different shades of brown. The front wall has a square window and a door. The two side walls have two windows each while the back wall is solid. The four walls support a two layer roof. The pole sub-assembly consists of eleven 1$\times$1 bricks, and one 1$\times$6 brick on top.

\begin{figure}[!htb]
\center{\includegraphics[width=0.65\textwidth]{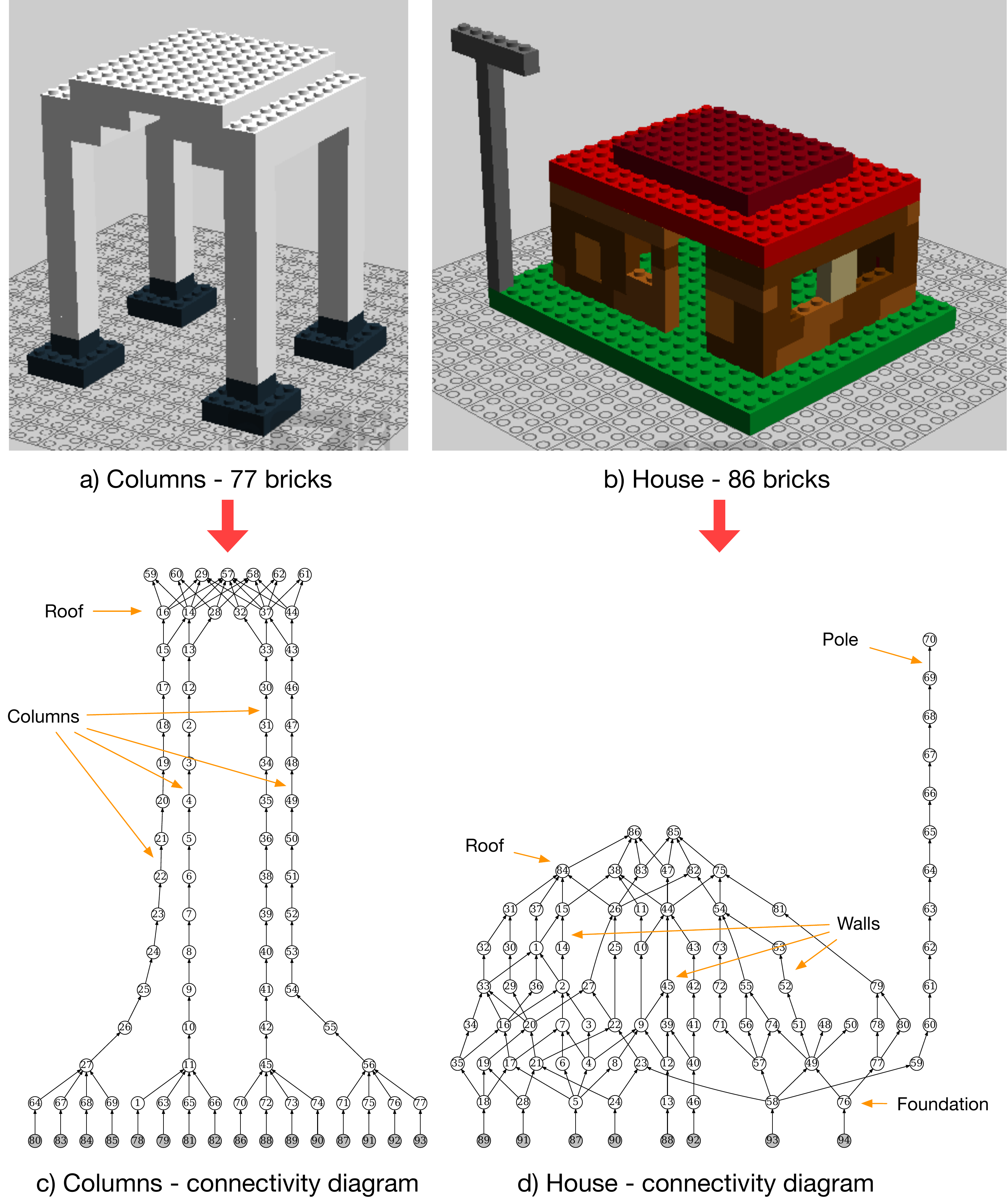}}
\caption{\label{fig:cad_models}LEGO CAD models and their inferred connectivity diagrams.}
\end{figure}

\subsection{Connectivity Diagrams}

The generated connectivity diagrams for the two CAD models are shown in Figure~\ref{fig:cad_models}(c-d). The Columns model was designed with regularity and symmetry in mind. These features are explicit in its connectivity diagram in Figure~\ref{fig:cad_models}(c) with the four long strands representing the columns, and the dense layer on top representing the roof. On the other hand, the House model was designed with asymmetry and irregularity in mind. These features are clearly visible in its connectivity diagram in Figure~\ref{fig:cad_models}(d). In the House connectivity diagram, the walls are irregular and asymmetric with gaps representing the windows and the door. The pole is represented by the long strand.

Typically, LEGO models come with assembly instructions, or \textit{build instructions}~\cite{legomanuals}. These instructions are, most likely, made for human enjoyment and therefore must be intuitive. One natural way to organize these instructions is by sub-assemblies such that humans can relate to the structure they are constructing (e.g., a house). In some cases, these sub-assemblies are obvious. For example, the column and roof sub-assemblies in the Columns connectivity diagram (Fig.~\ref{fig:cad_models}(c)) are easily distinguishable and therefore can be decomposed into a reasonable build plan. However, there are other cases when these sub-assemblies are not obvious. For example, decomposing the interlinked wall sub-assembly in the House model (Fig.~\ref{fig:cad_models}(d)) into a reasonable plan is not trivial. For both examples, our plan generation pipeline on string diagrams can be used to generate sequential and highly parallel assembly plans.


\subsection{Plan Generation}
In this paper, the quality of an assembly plan is determined in terms of how many operations can be executed in parallel, rather than on maximizing human enjoyment. For each LEGO CAD model, we generate two schedules. A sequential schedule is generated by topologically sorting the connectivity diagram, and a parallel schedule is generated with the algorithms introduced in Section~\ref{sec:plans}. Figure~\ref{fig:sequential_plans} shows the sequential schedules generated for the Columns and the House models. 

The column symmetry in the Columns model allows the sequential schedule to expose some parallelism as shown in Figure~\ref{fig:sequential_plans}(a). If several workers are available, this natural parallelism can be exploited to reduce the time-to-build. The sequential schedule for the House model, as shown in Figure~\ref{fig:sequential_plans}(b), exposes very little parallelism due to the asymmetry and irregularity in the model. These two examples help us illustrate the inherent limitations of sequential schedules.

\vspace{-0.2in}
\begin{figure}[!h]
\center{\includegraphics[width=0.85\textwidth]{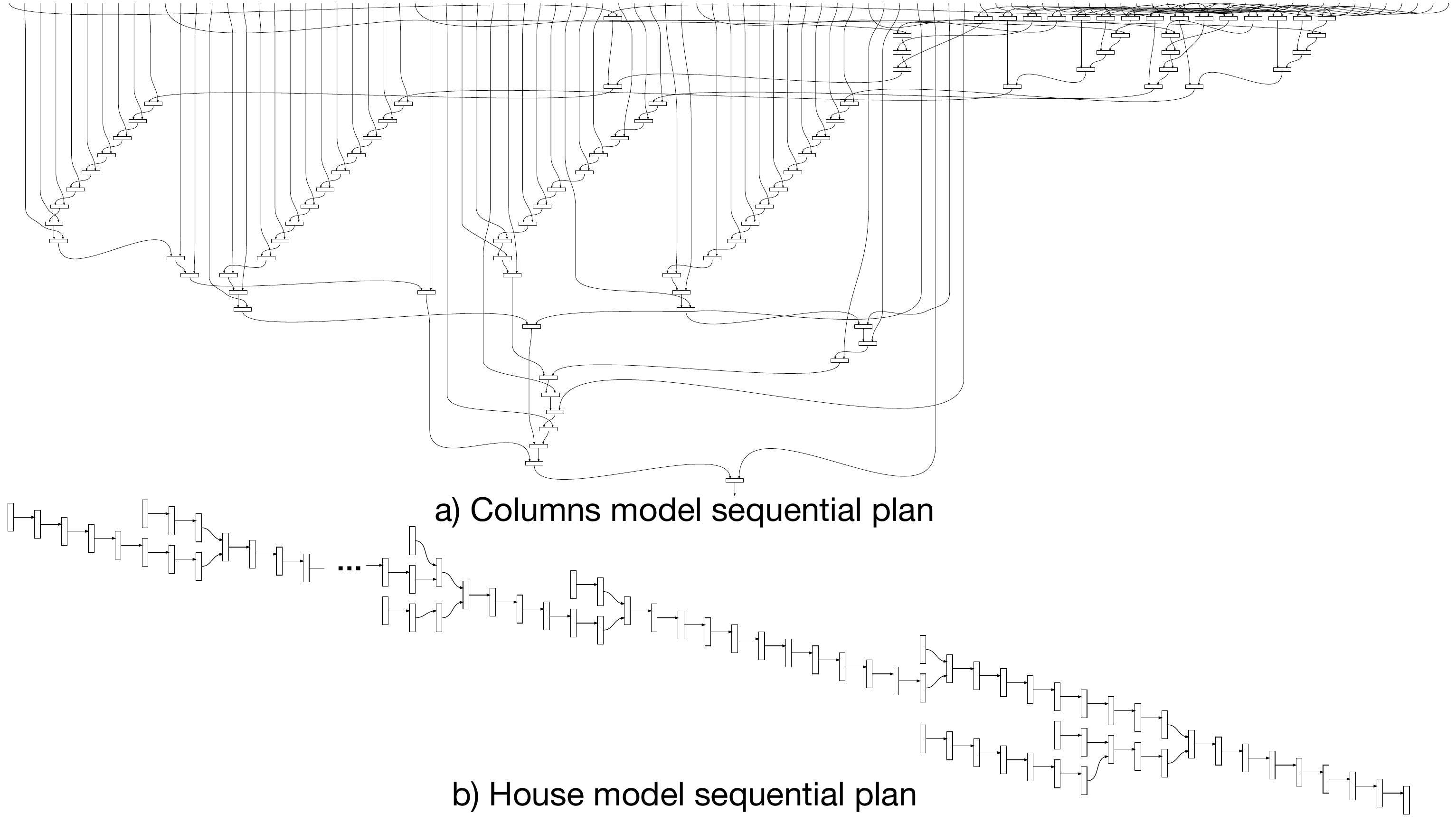}
\caption{\label{fig:sequential_plans}Sequential plans generated for the two LEGO CAD models.}}
\end{figure}
\vspace{-0.2in}

Using the parallel plan generation algorithms described in Section~\ref{sec:plans}, our framework exposes higher levels of parallelism as shown in Figure~\ref{fig:parallel_plans}.  Here, the black boxes corresponding to a partitioning of the connectivity diagram are shown by the labeled black boxes. The number in each black box represents the number of bricks within the black box sub-plan. The execution schedules are derived from these parallel plans. Even in the case of a purely sequential plan of the black boxes (represented by the width of the black boxes), it can be observed that the amount of parallelism is higher compared to the sequential schedules in Figure~\ref{fig:sequential_plans}. In particular, the House model's parallel plan in Figure~\ref{fig:parallel_plans}(b) exposes much higher parallelism when compared to the its sequential plan in Figure~\ref{fig:sequential_plans}(b) consisting of a stairs configuration with minimum parallelism.

\begin{figure}[!htb]
\center{\includegraphics[width=0.85\textwidth]{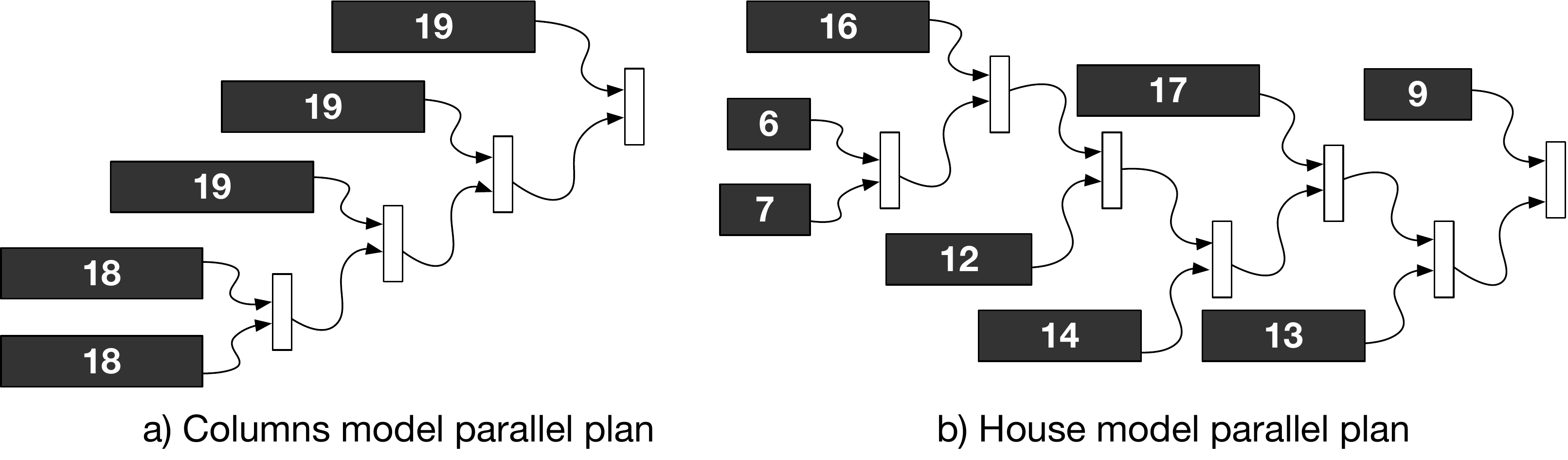}
\caption{\label{fig:parallel_plans}Parallel plans generated for the two LEGO CAD models. The width of the black boxes represent a sub-plan (i.e., a stairs configuration).}}
\end{figure}

\subsection{Simulated Schedule Execution}
The schedule and the number of workers determines how many operations can be executed in parallel. This parallelism can be visualized in the simulation when multiple sub-assemblies are constructed at the same time. Figure \ref{fig:assembly_process} shows a time-lapse of the parallel assembly process for the Columns and House models with an unlimited number of workers. The top row shows the assembly process at an early stage after a few blocks were already added. Note the parallel construction areas of the assemblies highlighted by the red squares. Furthermore, the main construction area where pieces and sub-assemblies will eventually be connected to each other can be recognized. Sub-assemblies that are connected to the ground are not constructed in separate construction areas but at their final position in the main construction area. This prevents excessive shifting of assemblies. While all the sub-assemblies are easily distinguishable for the Columns model, the House's main construction area already consists of three sub-assemblies which can be recognized by the three separate walls on top of the House's base. These construction areas are based on the black boxes of the plans and their representation in the corresponding schedule. The (sub-)schedule for a black box may contain some level of parallelism like the Columns model. Besides the parallel assembly of the columns the derived schedule allows for a parallel assembly of the roof in three separate construction areas. The middle row shows the half completed assemblies. The structures are more advanced, and some sub-assemblies have already been connected to other assemblies. The bottom row shows the fully assembled Columns and House LEGO models.

\begin{figure}[!h]
\center{\includegraphics[width=0.75\textwidth]{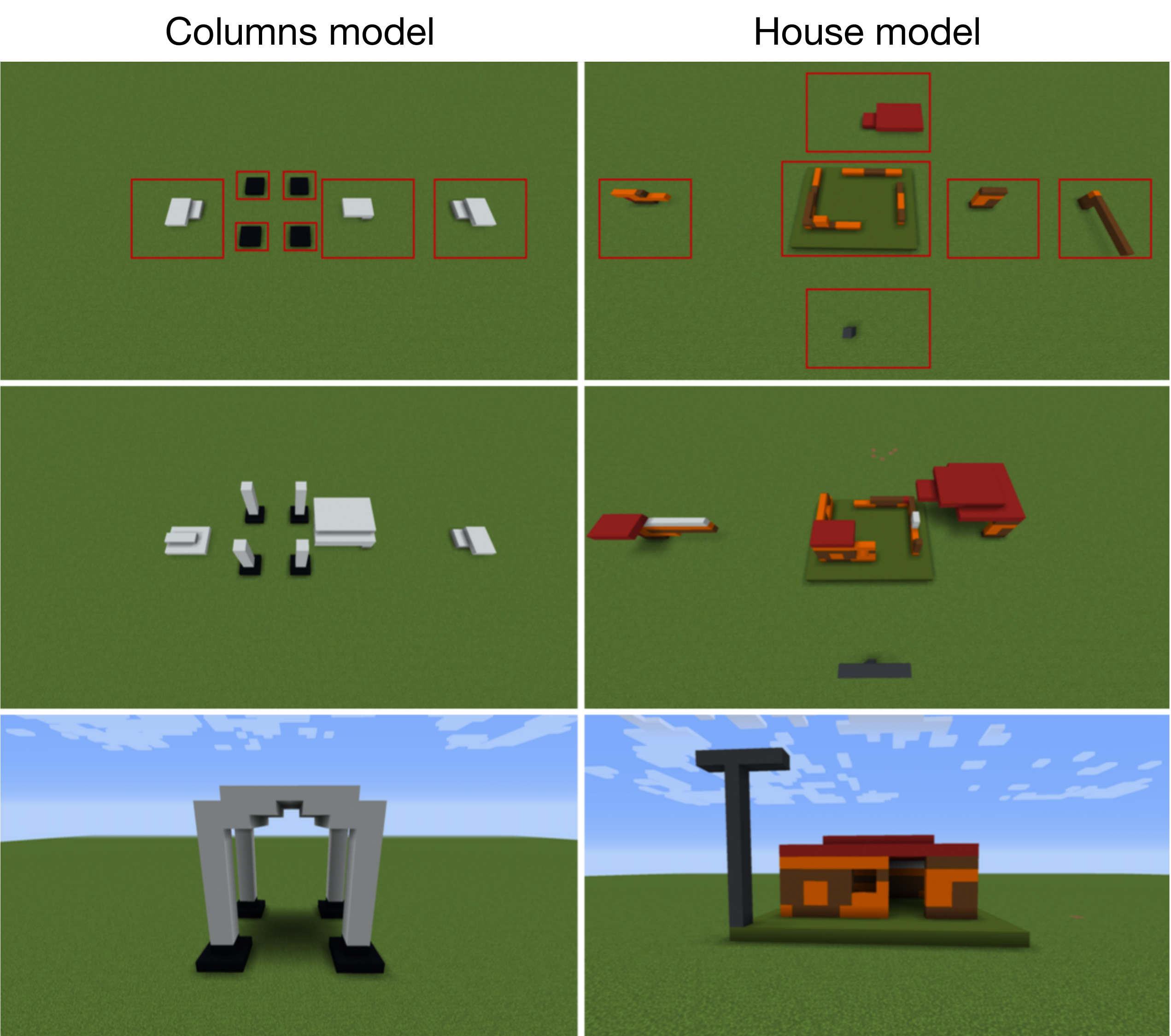}
\caption{\label{fig:assembly_process}Time-lapse of the execution of parallel plans. Black box sub-assemblies are built simultaneously on different construction areas highlighted by the red squares.}}
\end{figure}

To assess the quality of the generated plans we simulate the assembly process for the sequential and the parallel schedules for both models. The simulation is run with a varying number of workers – 1, 2, 4, 8 and 16. For each configuration we measured the total number of steps it takes to build the assembly and the average occupancy ratio of the workers. Table \ref{tab:resultData} summarizes the results. Our baselines are simulations run with a single worker. While pieces can be snapped to the ground or to an already existing assembly, two single pieces can also be combined to a new sub-assembly. In the latter case we must first connect a LEGO brick to the ground and this occupies a worker for a time step. For this reason, the time-to-built for the parallel schedule with a single worker is higher than the one for the sequential schedule.
Unsurprisingly, the parallel schedule with more than one worker always requires fewer steps to complete the assembly while maintaining a higher occupancy rate than its sequential counterpart. The performance difference between the sequential and parallel schedules for the Columns model is rather small because of its architecture. Its sequential plan (Figure~\ref{fig:sequential_plans}(a)) already exposes some parallelism. However, our parallel schedule is able to exploit additional opportunities and provides slightly better performance over the sequential schedule.

\begin{table}[htb]
\centering
\caption{Execution time and worker occupancy for different schedules with varying number of workers.}
\begin{tabular}{@{\extracolsep{4pt}}lp{1.2cm}llll}
\toprule
 {} & {} & \multicolumn{2}{l}{Columns model}& \multicolumn{2}{l}{House model} \\
 {Schedule} & {Workers} & Steps & Occupancy & Steps & Occupancy\\
\midrule
\multirow{5}{*}{\makecell[l]{Sequential}}
& 1 & 92 & 1.00 & 93 & 1.00 \\
& 2 & 50 & 0.92 & 75 & 0.62 \\
& 4 & 33 & 0.70 & 69 & 0.34 \\
& 8 & 25 & 0.46 & 65 & 0.18 \\
& 16 & 23 & 0.25 & 65 & 0.09\\ 
\midrule
\multirow{4}{*}{\makecell[l]{Parallel}}
& 1 & 95 & 1.00 & 98 & 1.00 \\
& 2 & 50 & 0.95 & 55 & 0.89 \\
& 4 & 30 & 0.79 & 36 & 0.68 \\
& 8 & 19 & 0.63 & 26 & 0.47 \\
& 16 & 17 & 0.35 & 21 & 0.29\\
\bottomrule
\end{tabular}

\label{tab:resultData}
\end{table}

For the House model on the other hand, the parallel schedule yields significantly faster assembly times with an increasing number of workers; up to 3 times with 16 workers. Additionally, the workers are used to a higher capacity and the occupancy ratio deteriorates at a slower rate as the number of workers increases. This model shows that our algorithms are able to identify and exploit non-obvious parallelism in less regular and symmetric models.

\section{Conclusion and Future Work}
In this paper, we studied the use of string diagrams for assembly planning and developed a framework to demonstrate this approach in the LEGO domain. This new perspective gives us multiple advantages. First, it provides us with a powerful graphical calculus for reasoning about the assembly planning domain in category theory. Second, this formalism allows string diagrams to be easily manipulated within a programming framework to generate plans and schedules. This allows us to seamlessly interconnect the different disciplines involved in assembly planning. Third, with a novel hierarchical planning approach using black boxes, we demonstrated that the resulting plans expose high-degrees of parallelism that result in efficient assembly. Our \textsc{CompositionalPlanning} framework has several limitations that prescribe future research. (i) As a proof of concept, we focused on the most popular LEGO assembly operations. Since our framework is domain agnostic, extending to other domains is an important direction for future work. (ii) We implemented three planning and one scheduling algorithm. Implementing other algorithms will help us validate the full potential of string diagrams for assembly planning. (iii) The planning algorithms yield different string diagrams. Developing new algorithms to \textit{morph} a string diagram to another with different properties (e.g., more parallelism) is a challenging but very interesting research direction. 

%
%

\bibliographystyle{splncs04}
\bibliography{IEEEexample}

\end{document}